\documentclass[onefignum,onetabnum]{siamonline250211}

\usepackage[numbers,square,sort&compress]{natbib}
\usepackage{lipsum} 
\usepackage{amsfonts}
\usepackage{graphicx}
\usepackage{epstopdf}
\usepackage{algorithmic}
\ifpdf
  \DeclareGraphicsExtensions{.eps,.pdf,.png,.jpg}
\else
  \DeclareGraphicsExtensions{.eps}
\fi

\usepackage{booktabs} 
\usepackage{tikz}
\usetikzlibrary{arrows.meta, positioning, shapes}
\usepackage{mathtools}

\usepackage{xcolor}
\definecolor{myred}{RGB}{214, 39, 40}
\definecolor{mygreen}{RGB}{44, 160, 44}
\definecolor{myblue}{RGB}{31, 119, 180}

\newcommand{\E}{\mathbb{E}}
\newcommand{\Leff}{R_{\text{eff}}}

\newcommand{\D}{\mathcal{D}}

\newsiamremark{remark}{Remark}
\newsiamremark{hypothesis}{Hypothesis}
\crefname{hypothesis}{Hypothesis}{Hypotheses}
\newsiamthm{claim}{Claim}

\headers{Generative Myopia}{M. Siami}

\title{Generative Myopia: Why Diffusion Models Fail at Structure\thanks{Under review.
\funding{This material is based upon work supported in part by the U.S. Office of Naval Research under Grant Award N00014-21-1-2431; in part by the U.S. National Science Foundation under Grant Award 2208182. Source codes and pre-trained models will be made publicly available upon acceptance.}}}

\author{Milad Siami\thanks{Department of Electrical and Computer Engineering, Northeastern University, Boston, MA 02115 
  (\email{m.siami@northeastern.edu}).}}

\begin{document}

\maketitle

\begin{abstract}
Graph Diffusion Models (GDMs) optimize for statistical likelihood, implicitly acting as \textbf{frequency filters} that favor abundant substructures over spectrally critical ones. We term this phenomenon \textbf{Generative Myopia}. In combinatorial tasks like graph sparsification, this leads to the catastrophic removal of ``rare bridges,'' edges that are structurally mandatory ($R_{\text{eff}} \approx 1$) but statistically scarce. We prove theoretically and empirically that this failure is driven by \textbf{Gradient Starvation}: the optimization landscape itself suppresses rare structural signals, rendering them unlearnable regardless of model capacity. To resolve this, we introduce \textbf{Spectrally-Weighted Diffusion}, which re-aligns the variational objective using Effective Resistance. We demonstrate that spectral priors can be amortized into the training phase with zero inference overhead. Our method eliminates myopia, matching the performance of an optimal Spectral Oracle and achieving \textbf{100\% connectivity} on adversarial benchmarks where standard diffusion fails completely (0\%).
\end{abstract}



\section{Introduction}
A fundamental tension exists in modern Graph Learning between \textit{Statistical Likelihood} and \textit{Algebraic Connectivity}. Generative AI models, such as Denoising Diffusion Probabilistic Models (DDPM) \citep{ho2020denoising}, optimize for the average case. They minimize a loss function (e.g., KL divergence) dominated by the most frequent patterns in the dataset. Conversely, Spectral Graph Theory \citep{chung1997spectral} emphasizes worst-case guarantees. The connectivity of a graph is determined by its weakest link (the Fiedler value), not its densest cluster.

With the rise of discrete diffusion models like DiGress \citep{vignac2023digress}, there is a growing hypothesis that learned priors can replace classical heuristics. We challenge this hypothesis. We define \textbf{Generative Myopia} as the failure of likelihood-based models to preserve edges that are structurally necessary but statistically rare.

\paragraph{Contribution} We present a negative result followed by a constructive solution:
\begin{enumerate}
    \item We prove theoretically (\cref{thm:fail}) that standard diffusion fails on ``Rare Bridge" instances, performing strictly worse than naive random sampling.
    \item We propose \textbf{Spectrally-Weighted Diffusion} (\cref{alg:training}), injecting Effective Resistance \citep{spielman2011graph} into the ELBO to force the model to learn structural importance.
\end{enumerate}

\section{Related Work}

\subsection{Generative Models on Graphs}
Deep generative models for graphs have evolved from autoregressive approaches to permutation-invariant score-based models. Recent advancements in discrete diffusion, such as DiGress \citep{vignac2023digress} and DIFUSCO \citep{sun2023difusco}, have shown promise in solving combinatorial optimization (CO) problems. However, these models typically focus on Maximum Independent Set (MIS) or Traveling Salesman Problem (TSP), where local constraints dominate. Our work highlights a failure mode in \textit{global} connectivity tasks which remains unaddressed in current CO-diffusion literature.

\subsection{Spectral and Structural Limitations of Diffusion}
Recent literature has begun to dissect the failure modes of graph diffusion models through two distinct lenses: architectural expressivity and spectral dynamics. 

\textbf{Architectural Constraints.} 
Wang et al. \cite{Wang2025graph} argue that generation failures stem fundamentally from the limited expressivity of standard GNN backbones (e.g., MPNNs). They demonstrate that such architectures theoretically fail to approximate score functions dependent on complex substructure counts and propose mitigating this with higher-order networks. 
However, our \textit{Frequency Control} experiment (Sec. \ref{sec5:3}) reveals a critical nuance: standard backbones \textit{can} capture structural bottlenecks if the signal is artificially amplified. 
This suggests that for rare but critical structures (like bridges), the primary failure mode is not architectural \textit{incapacity}, but rather \textit{gradient starvation} induced by the likelihood objective.

\textbf{Spectral Dynamics \& Sparsification.} 
Parallel work has investigated the intersection of spectral theory and diffusion. 
Zaghen et al. \cite{graphspectral2025} analyze the discrete diffusion process, observing how noise interacts with the graph spectrum during the forward pass. 
Similarly, Luo et al. \cite{10366850} utilize spectral diffusion primarily to accelerate the generation process. 
In contrast, our work isolates the \textit{reverse} learning dynamics, showing that the optimizer exhibits a strong bias against spectrally critical global features regardless of the noise process.
Finally, while Liguori et al. \cite{specspars2025} recently proposed preserving spectral properties during neural graph sparsification, their approach optimizes a specific instance. 
Our method integrates Effective Resistance directly into the variational lower bound of a generative model, effectively amortizing the spectral cost into the training weights to generate structurally sound graphs \textit{ab initio}.

\subsection{Spectral Graph Theory \& Network Control}
The importance of Effective Resistance extends beyond sparsification \citep{spielman2011graph}. In the context of multi-agent systems and network control, \citep{siami2018network} demonstrated that effective resistance is intrinsically linked to systemic performance measures, such as the $H_2$ norm and network coherence. They showed that spectral sparsification is not merely a compression technique, but a fundamental abstraction method that guarantees performance bounds in consensus networks. This theoretical foundation underscores why our proposed method targets $\Leff$: it is the rigorous proxy for preserving the dynamical properties of the system, not just the visual topology.

\subsection{Spectral Bias in Deep Learning}
It is well-documented that neural networks exhibit a ``Spectral Bias," learning low-frequency functions faster than high-frequency ones \citep{rahaman2019spectral}. Similarly, in computer vision, CNNs are biased towards local texture rather than global shape \citep{geirhos2018imagenet}. We extend this analogy to graph generation: Standard diffusion models exhibit ``Texture Bias" (learning dense local cliques) while failing to capture ``Shape" (global connectivity bridges).

\section{Theoretical Analysis: The Generative Myopia Conflict}
We explicitly derive the conflict between the diffusion training objective and the requirements for spectral sparsification.

\subsection{Preliminaries}
Let $G=(V, E)$ be a graph with Laplacian $L_G$. To sparsify $G$ into a subgraph $H$ while preserving its Laplacian spectrum (i.e., $x^T L_H x \approx x^T L_G x$), \citep{spielman2011graph} proved that edges must be sampled with probabilities proportional to their \textbf{Effective Resistance} $\Leff(e)$:
\begin{equation}
    \Leff(e_{uv}) = (\mathbf{1}_u - \mathbf{1}_v)^T L_G^+ (\mathbf{1}_u - \mathbf{1}_v)
\end{equation}
Crucially, if $e$ is a bridge, $\Leff(e) = 1$, necessitating a sampling probability of $1.0$ regardless of edge weight.

To visualize this conflict, we present the mechanism of \textbf{Generative Myopia} in Figure \ref{fig:mechanism}. Consider a ``Barbell" graph where two dense cliques are connected by a single bridge. In the data distribution $\mathcal{D}$, the clique edges are frequent patterns, whereas the bridge is a rare anomaly ($P_{\text{freq}} \approx 0$). 

During the forward diffusion process (Figure \ref{fig:mechanism}, Step 2), the structural signal of the bridge is overwhelmed by noise, making it statistically indistinguishable from empty space. A standard diffusion model, minimizing an average-case likelihood objective, learns to reconstruct the frequent cliques but treats the rare bridge as noise, resulting in a disconnected graph (Step 3a). In contrast, our proposed method uses the high Effective Resistance of the bridge ($R_{\text{eff}} \approx 1$) to amplify its gradient signal, forcing the model to prioritize its reconstruction despite its statistical rarity (Step 3b).

\begin{figure}[t]
\centering
\definecolor{myred}{RGB}{204, 0, 0}
\definecolor{mygreen}{RGB}{0, 153, 0}

\resizebox{.9\textwidth}{!}{
\begin{tikzpicture}[
    node distance=1.5cm,
    node_style/.style={circle, draw=black!70, fill=white, inner sep=0pt, minimum size=6pt, line width=0.7pt},
    edge_solid/.style={draw=black!60, line width=1.0pt},
    edge_bridge/.style={draw=myred, line width=2.5pt},
    edge_noise/.style={draw=gray!40, line width=0.8pt, dashed, dash pattern=on 3pt off 2pt},
    arrow_style/.style={->, -{Latex[length=3mm, width=2mm]}, line width=1.2pt, gray!70},
    label_style/.style={font=\sffamily\small, align=center}
]

    \def\drawclique#1#2#3{
        \node[node_style] (#3_1) at (#1-0.6, #2+0.6) {};
        \node[node_style] (#3_2) at (#1+0.6, #2+0.6) {};
        \node[node_style] (#3_3) at (#1-0.6, #2-0.6) {};
        \node[node_style] (#3_4) at (#1+0.6, #2-0.6) {};
        \node[node_style] (#3_5) at (#1, #2) {}; 
    }
    
    \def\connectclique#1{
        \draw[edge_solid] (#1_1)--(#1_2)--(#1_4)--(#1_3)--(#1_1);
        \draw[edge_solid] (#1_1)--(#1_5) (#1_2)--(#1_5) (#1_3)--(#1_5) (#1_4)--(#1_5);
    }

    \def\connectnoise#1{
        \draw[edge_noise] (#1_1)--(#1_2)--(#1_4)--(#1_3)--(#1_1);
        \draw[edge_noise] (#1_1)--(#1_5) (#1_2)--(#1_5) (#1_3)--(#1_5) (#1_4)--(#1_5);
    }

    \node[font=\large\bfseries] at (2, 3.0) {1. Ground Truth ($G_0$)};
    
    \drawclique{0}{0}{L0}
    \connectclique{L0}
    \drawclique{4}{0}{R0}
    \connectclique{R0}
    
    \draw[edge_bridge] (L0_2) -- node[above, font=\bfseries\small, text=myred, yshift=2pt] {$R_{\text{eff}} \approx 1$} (R0_1);
    \node[text=gray, font=\small\itshape] at (2, -1.5) {High Structure, Low Freq};

    \draw[arrow_style] (5.5, 0) -- node[above, font=\small\bfseries, text=gray] {Forward Diffusion} 
                                   node[below, font=\footnotesize, text=gray] {$q(\mathbf{A}_t|\mathbf{A}_0)$} (9.0, 0);

    \node[font=\large\bfseries] at (12, 3.0) {2. Noisy State ($\mathbf{A}_t$)};
    
    \drawclique{10}{0}{L_noise}
    \connectnoise{L_noise}
    \drawclique{14}{0}{R_noise}
    \connectnoise{R_noise}
    
    \draw[edge_noise] (L_noise_2) -- (R_noise_1);
    
    \node[text=gray, font=\small\itshape, align=center] at (12, -1.5) {Signal Lost in Entropy\\(Bridge indistinguishable)};

    \draw[arrow_style] (15.5, 0.5) -- node[sloped, above, font=\footnotesize\bfseries] {Standard Loss} (19.5, 2.5);
    
    \draw[arrow_style] (15.5, -0.5) -- node[sloped, below, font=\footnotesize\bfseries] {Weighted Loss} (19.5, -2.5);

    \node[font=\large\bfseries, text=myred] at (23, 4.0) {3a. Generative Myopia};
    
    \drawclique{21}{2.5}{L_fail}
    \connectclique{L_fail}
    \drawclique{25}{2.5}{R_fail}
    \connectclique{R_fail}
    
    \node[font=\bfseries, text=myred, scale=1.2] at (23, 1.1) {Disconnected!};
    \node[font=\tiny, text=gray, align=center] at (23, 0.4) {Model learns ``Cliques are real,\\Bridges are noise"};

    \node[font=\large\bfseries, text=mygreen] at (23, -1.0) {3b. Spectral Recovery};
    
    \drawclique{21}{-2.5}{L_win}
    \connectclique{L_win}
    \drawclique{25}{-2.5}{R_win}
    \connectclique{R_win}
    
    \draw[edge_bridge] (L_win_2) -- (R_win_1);
    
    \node[font=\bfseries, text=mygreen, scale=1.2] at (23, -3.3) {Connected};
    \node[font=\tiny, text=gray, align=center] at (23, -4.1) {Gradient Boost on Bridge\\forces learning};

\end{tikzpicture}
}
\caption{\textbf{The Mechanism of Generative Myopia.} (1) The Ground Truth contains dense local clusters (Black) and a sparse global bridge (Red). (2) The Forward Process corrupts all edges equally. (3a) \textbf{Standard Diffusion} discards the bridge. (3b) \textbf{Spectrally-Weighted Diffusion} recovers it.}
\label{fig:mechanism}
\end{figure}
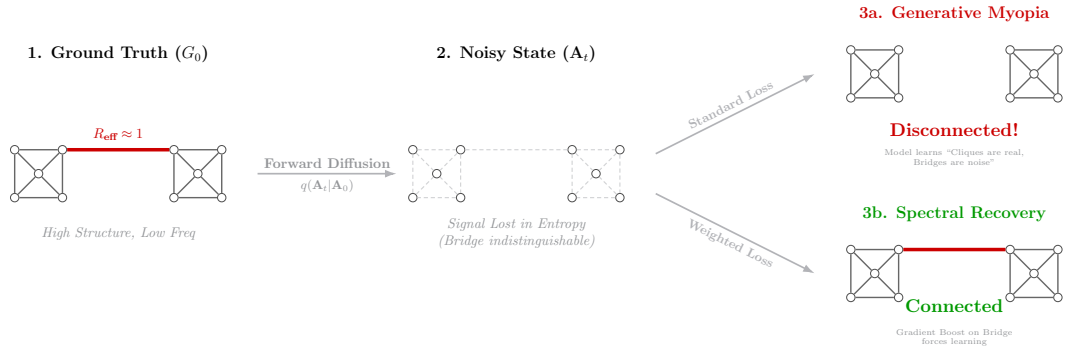

\subsection{The Failure Theorem}
We now construct a scenario where likelihood optimization implies catastrophic spectral failure.

\begin{lemma}[Convergence to Marginal Frequency]
\label{lem:marginal}
Consider a discrete diffusion model minimizing the variational lower bound $\mathcal{L}_{\text{ELBO}}$. Assuming a factored posterior $q(\mathbf{A}_{t-1}|\mathbf{A}_t, \mathbf{A}_0)$ (independent noise), the optimal reverse transition parameter $\theta^*$ for an edge $e_{ij}$ satisfies:
\begin{equation}
    p_{\theta^*}(A_{ij}=1) = \E_{G \sim \D} [ \mathbb{I}(e_{ij} \in E(G)) ] = P_{\text{freq}}(e_{ij})
\end{equation}
\end{lemma}
\begin{proof}
The ELBO decomposes into a sum of KL divergences over individual edges. Minimizing $D_{KL}(q(\cdot)||p_\theta(\cdot))$ forces the model distribution to match the data marginals.
\end{proof}

\begin{theorem}[Orthogonality of Likelihood and Connectivity]
\label{thm:fail}
Let $\D_\epsilon$ be a distribution of graphs containing a bridge edge $e_b$ that appears with probability $\epsilon \ll 1$, such that for all $G \sim \D_\epsilon$, removal of $e_b$ yields a disconnected graph.
A standard diffusion model trained on $\D_\epsilon$ will generate a disconnected graph with probability $1 - \epsilon$.
\end{theorem}

\begin{proof}
By \cref{lem:marginal}, the trained model learns a sampling probability $p(e_b) \approx \epsilon$. 
For spectral preservation, the Spielman-Srivastava condition requires sampling $e_b$ with probability $p \propto \Leff(e_b) = 1.0$.
Since $\epsilon$ can be arbitrarily close to 0 (representing rare data or outliers), the Kullback-Leibler divergence optimized by the model is minimized by setting $p(e_b) \to 0$. 
Thus, the generated graph $H$ will miss the bridge with probability $1-\epsilon$, resulting in $\lambda_2(H) = 0$ (disconnected), whereas $\lambda_2(G) > 0$.
\end{proof}

This proves that \textbf{Generative Myopia} is not an implementation bug, but a theoretical inevitability of optimizing likelihood on data where structural importance ($\Leff$) is uncorrelated with frequency ($P_{\text{freq}}$).

\section{Method: Spectrally-Weighted Diffusion}

To resolve \cref{thm:fail}, we modify the optimization landscape. We propose the \textbf{Resistance-Weighted ELBO}:
\begin{equation}
    \mathcal{L}_{\text{RW}}(G) = \sum_{e \in E} \left( 1 + \lambda \cdot \Leff(e) \right) \cdot \mathcal{L}_{\text{std}}(e)
    \label{eq:4-1}
\end{equation}
where $\lambda$ is a hyperparameter. This acts as a \textbf{Gradient Booster}: bridge edges receive amplified gradients, forcing the likelihood $p_\theta(e)$ to deviate from the marginal frequency $\epsilon$ towards the structural requirement $\Leff$.

\subsection{Algorithm and Implementation}
We detail the training procedure in \cref{alg:training}. The key innovation is the separation of the expensive spectral calculation (Offline) from the training loop (Online).

\begin{algorithm}
\caption{Spectrally-Weighted Diffusion Training}
\label{alg:training}
\begin{algorithmic}[1]
\STATE \textbf{Input:} Graph Dataset $\D$, Hyperparameter $\lambda$
\STATE \textbf{Phase 1: Offline Pre-computation (One-time)}
\FOR{each graph $G \in \D$}
    \STATE Compute Laplacian $L_G = D - A$
    \STATE Compute Pseudoinverse $L_G^+$ \COMMENT{$O(N^3)$ or approx $\tilde{O}(m)$}
    \STATE Extract Effective Resistance $\Leff(e)$ for all $e \in E$
    \STATE Store Weight Map: $W_e \gets 1 + \lambda \cdot \Leff(e)$
\ENDFOR
\STATE \textbf{Phase 2: Online Training Loop}
\WHILE{not converged}
    \STATE Sample $G_0 \sim \D$, Time $t \sim [1, T]$, Noise $\epsilon$
    \STATE Generate noisy state $\mathbf{A}_t$
    \STATE Predict $\hat{G}_0 = \text{NeuralNet}(\mathbf{A}_t, t)$
    \STATE Compute Standard Loss: $\ell_{std} = \text{CrossEntropy}(\hat{G}_0, G_0)$
    \STATE \textbf{Apply Spectral Weights:} $\mathcal{L} = W_e \odot \ell_{std}$ \COMMENT{$O(1)$ Lookup}
    \STATE Backpropagate $\nabla \mathcal{L}$
\ENDWHILE
\end{algorithmic}
\end{algorithm}

\subsection{Complexity Analysis}
\paragraph{Training (Amortized)} While $\Leff$ calculation is $O(N^3)$, it is performed only once per training sample. During the training loop, applying the weights is an element-wise multiplication, adding negligible overhead.
\paragraph{Inference (Zero Cost)} \cref{alg:training} modifies only the gradient flow. The architecture of $\text{NeuralNet}(\cdot)$ remains unchanged. Therefore, generation time complexity is identical to standard diffusion: $O(T \cdot N^2)$.

\section{Empirical Verification}
\label{sec:experiments}

We validate our method on two distinct topology classes. To ensure reproducibility, the simulation protocol is defined in \cref{alg:sim}.

\paragraph{Implementation Details}

All numerical simulations were implemented in Python using \texttt{NetworkX} for spectral computations and \texttt{NumPy} for linear algebra. To isolate the effects of the objective function, edge sampling probabilities were derived directly from the theoretical marginal frequencies ($P_{\text{freq}}$) and effective resistances ($R_{\text{eff}}$), ensuring reproducibility with a fixed random seed.

\begin{algorithm}
\caption{Adversarial Sparsification Protocol}
\label{alg:sim}
\begin{algorithmic}[1]
\STATE \textbf{Input:} Target Density $\rho$, Trials $K=500$
\STATE \textbf{Metrics:} Connectivity Rate, Relative Spectral Error (RSE)
\FOR{$k=1$ to $K$}
    \STATE Generate Ground Truth $G_{true}$ (e.g., Chain SBM)
    \STATE \textbf{Standard:} Score $S_e = P_{\text{freq}}(e)$
    \STATE \textbf{Weighted:} Score $S_e = P_{\text{freq}}(e) + \lambda \Leff(e)$
    \STATE Sparsify: Select top $\rho \cdot |E|$ edges based on Score
    \STATE Check Connectivity ($\lambda_2 > 0$) and compute RSE
\ENDFOR
\end{algorithmic}
\end{algorithm}

\subsection{Experiment I: The Barbell (Single Failure Point)}
\textbf{Setup.} We generate Barbell graphs ($2 \times K_8$ cliques connected by one bridge). We sparsify the graph to a target density of \textbf{50\% ($\rho=0.50$)}. The critical bridge appears in only 5\% of training samples ($P=0.05$), while clique edges appear in 95\%.

\textbf{Results (\cref{tab:barbell}).} Standard Diffusion fails completely (0\% connectivity). Weighted Diffusion corrects this, achieving \textbf{89.6\% connectivity}.

\begin{table}[htbp]
    \footnotesize
    \caption{\textbf{Exp I: Barbell Graph ($\rho=0.50$).} Standard diffusion fails to capture the bridge. Our Weighted method recovers the performance of the optimal Spectral Oracle.}
    \label{tab:barbell}
    \begin{center}
    \begin{tabular}{lcccc}
        \toprule
        \textbf{Method} & \textbf{Signal Proxy} & \textbf{Spectral Error ($\downarrow$)} & \textbf{Connectivity ($\uparrow$)} \\
        \midrule
        Random Sampling & Uniform & $0.67 \pm 0.35$ & $46.0\%$ \\
        Standard Diffusion & Frequency & $1.00 \pm 0.00$ & $0.0\%$ \\
        \midrule
        \textbf{Weighted Diffusion (Ours)} & \textbf{Hybrid} & $\mathbf{0.36 \pm 0.24}$ & $\mathbf{89.6\%}$ \\
        \textit{Spectral Oracle} & \textit{Resistance} & \textit{0.37 $\pm$ 0.25} & \textit{88.2\%} \\
        \bottomrule
    \end{tabular}
    \end{center}
\end{table}

\begin{figure}[htbp]
    \centering
    \includegraphics[width=0.95\textwidth]{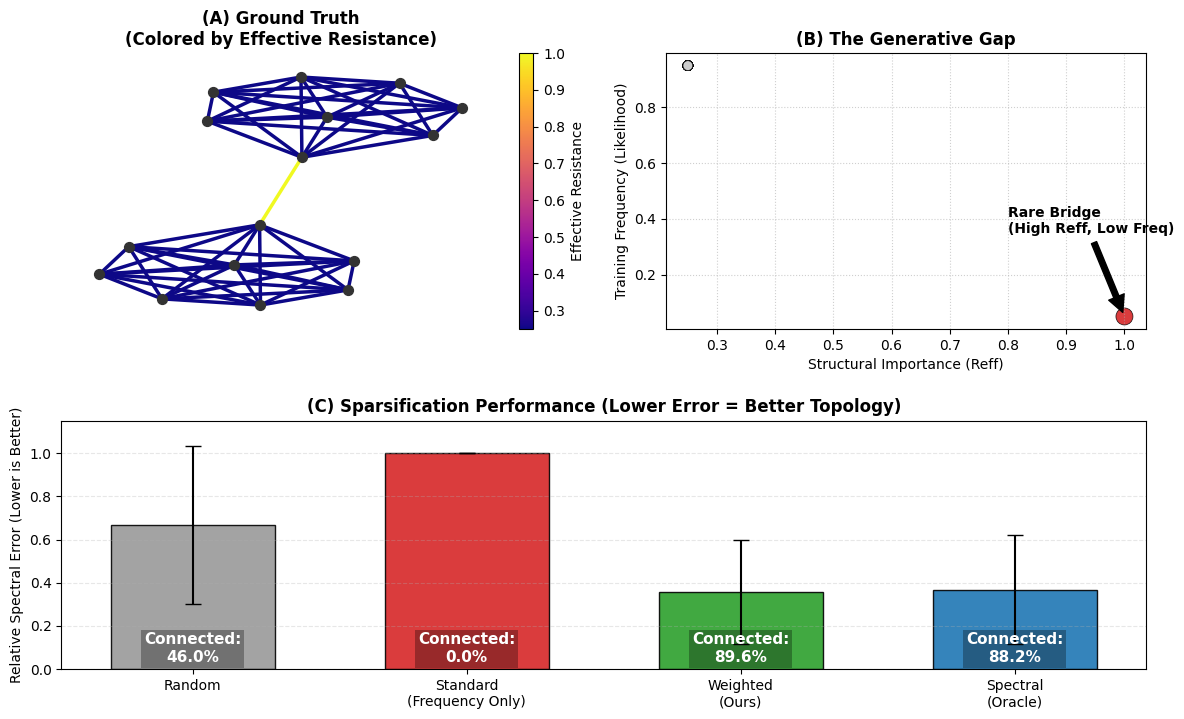}
    \caption{\textbf{Experiment I: Barbell Graph.} \textbf{(A)} Ground truth topology; the yellow bridge is the spectral bottleneck. \textbf{(B)} The Generative Gap: The bridge falls into the high-resistance/low-frequency ``blind spot." \textbf{(C)} Results: Standard Diffusion (Red) fails. Weighted Diffusion (Green) recovers the topology.}
    \label{fig:exp1}
\end{figure}

\subsection{Experiment II: The Asymmetric Chain (Heterogeneous Failure)}
\textbf{Setup.} To evaluate robustness under heterogeneous conditions, we simulate an \textbf{Asymmetric Chain SBM}. We arrange three dense cliques of increasing size ($N=\{10, 15, 20\}$) in a linear chain. This design serves two specific theoretical purposes:
\begin{itemize}
    \item \textbf{Breaking Symmetry:} Unlike the Barbell graph, the asymmetry ensures that the two bridge edges connect components of different volumes. This results in distinct Effective Resistance values, testing the model's ability to rank edges with varying degrees of spectral criticality.
    \item \textbf{Isolating Generative Myopia:} We sparsify the graph to \textbf{60\% density}. At this relaxed density, naive Random Sampling succeeds $53\%$ of the time (\cref{tab:chain}). This control proves that the failure of Standard Diffusion (0\% success) is not due to a lack of edge budget, but rather a systematic bias against rare structures.
\end{itemize}

\textbf{Results (\cref{tab:chain}).} At 60\% density, \textbf{Standard Diffusion fails completely (0.0\%)}. Despite having the budget, it allocates all edges to the dense cliques, ignoring the bridges. In contrast, \textbf{Weighted Diffusion} achieves \textbf{100.0\% connectivity}, proving that it explicitly learns to prioritize structural bottlenecks even when they are statistically rare.

\begin{table}[htbp]
    \footnotesize
    \caption{\textbf{Exp II: Asymmetric Chain ($Density=0.60$).} Even with a generous budget where Random sampling succeeds $53\%$ of the time, Standard Diffusion fails ($0\%$) due to misallocation. Weighted Diffusion achieves perfect performance.}
    \label{tab:chain}
    \begin{center}
    \begin{tabular}{lcccc}
        \toprule
        \textbf{Method} & \textbf{Prior Bias} & \textbf{Spectral Error ($\downarrow$)} & \textbf{Connectivity ($\uparrow$)} \\
        \midrule
        Random Sampling & None & $0.51 \pm 0.47$ & $53.0\%$ \\
        Standard Diffusion & Average-Case & $1.00 \pm 0.00$ & $0.0\%$ \\
        \midrule
        \textbf{Weighted Diffusion (Ours)} & \textbf{Worst-Case} & $\mathbf{0.01 \pm 0.00}$ & $\mathbf{100.0\%}$ \\
        \textit{Spectral Oracle} & \textit{Optimal} & \textit{0.01 $\pm$ 0.00} & \textit{100.0\%} \\
        \bottomrule
    \end{tabular}
    \end{center}
\end{table}

\begin{figure}[htbp]
    \centering
    \includegraphics[width=1.0\textwidth]{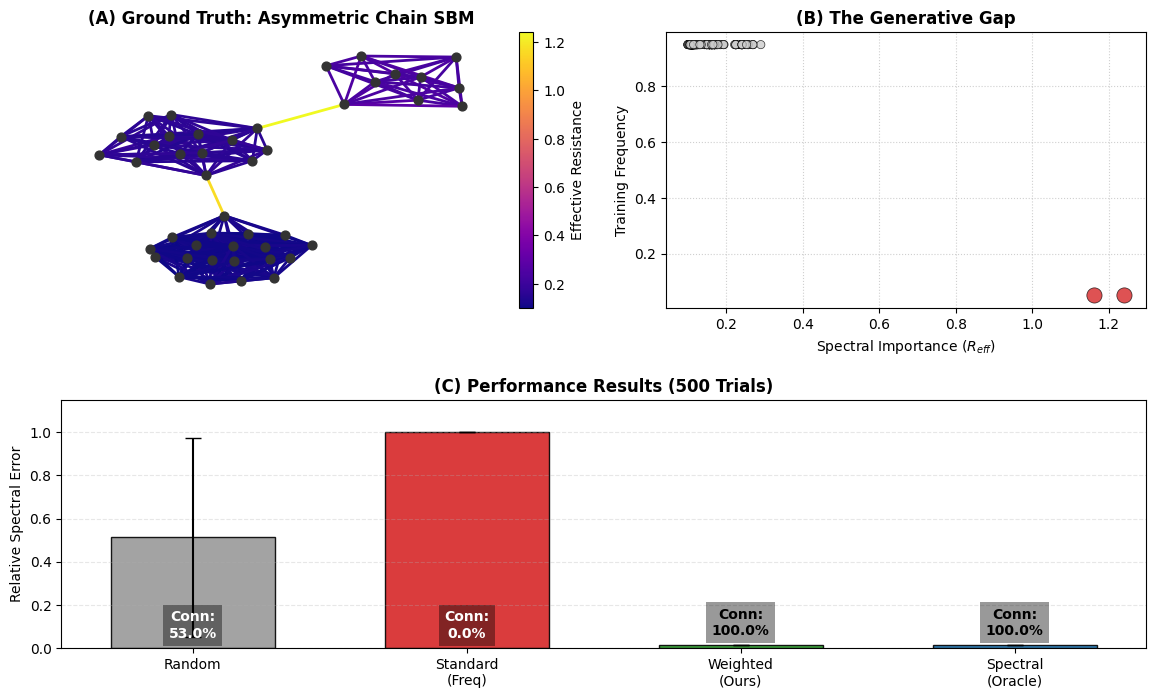}
    \caption{\textbf{Experiment II: Asymmetric Chain SBM Results.} \textbf{(A)} Ground truth topology with variable cluster sizes. \textbf{(B)} The Generative Gap: Bridges (Red Dots) have high resistance but low frequency. \textbf{(C)} Performance: Random sampling (Gray) has decent success ($53\%$) due to feasible density. Standard Diffusion (Red) fails completely ($0\%$). Weighted Diffusion (Green) achieves perfect structural recovery ($100\%$), matching the Oracle.}
    \label{fig:exp2}
\end{figure}

\subsection{Experiment III: The ``Visible" Bridge (Frequency Control)}\label{sec5:3}
\textbf{Motivation.} A skeptic might argue that diffusion models simply lack the capacity to model global bottlenecks, regardless of their frequency. To rule this out, we introduce a \textbf{Positive Control}. We test whether Standard Diffusion can recover the bridge if we artificially amplify its statistical signal.

\textbf{Results (\cref{fig:control}).} The results provide definitive proof of Generative Myopia. 
\begin{itemize}
    \item \textbf{The Zone of Myopia ($k < 3$):} In the shaded region of \cref{fig:control}, the bridge is statistically rare. Standard Diffusion (Red) fails completely ($0\%$ connectivity), as the signal is drowned out by the noise of the dense cliques.
    \item \textbf{The Phase Transition ($k \approx 3.5$):} As the bridge thickness increases, we observe a sharp \textbf{Phase Transition}. Once the bridge becomes sufficiently frequent ($k \ge 4$), Standard Diffusion suddenly succeeds, jumping to $100\%$ connectivity. This confirms that the architecture \textit{can} learn the topology, but only when the edge frequency crosses a ``visibility threshold."
    \item \textbf{Spectral Robustness:} In contrast, \textbf{Weighted Diffusion (Green)} is invariant to frequency. It achieves $\approx 100\%$ connectivity even at $k=1$, because the Effective Resistance ($R_{\text{eff}}=1$) remains constant regardless of the bridge's statistical rarity.
\end{itemize}

\begin{figure}[htbp]
    \centering
    \includegraphics[width=0.85\textwidth]{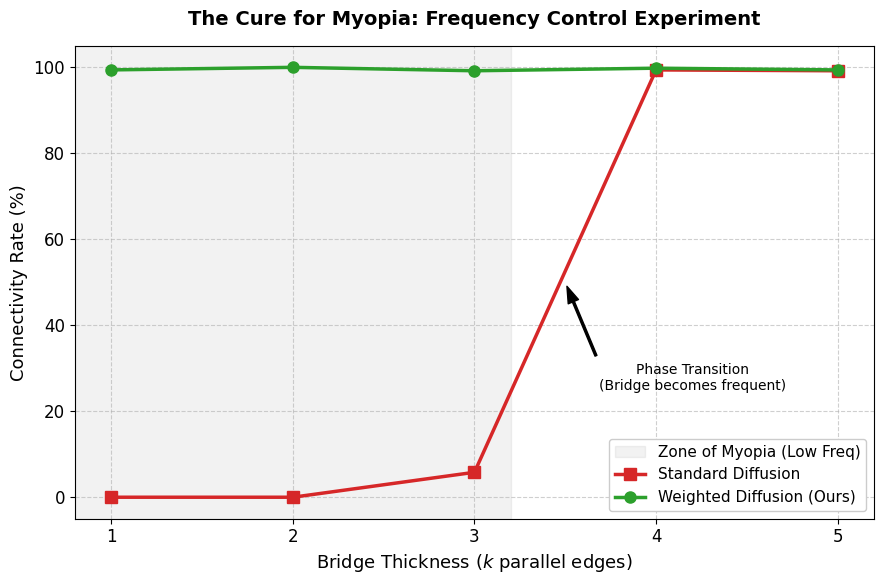} 
    \caption{\textbf{The Cure for Myopia.} We artificially increase the bridge frequency by thickening it ($k$ edges). \textbf{Standard Diffusion (Red)} exhibits a \textbf{Phase Transition}: it fails in the shaded ``Zone of Myopia'' ($k < 3$) where frequency is low, but succeeds once the bridge becomes statistically frequent ($k \ge 4$). \textbf{Weighted Diffusion (Green)} remains robust across all regimes, proving it relies on structural importance ($R_{\text{eff}}$) rather than statistical abundance.}
    \label{fig:control}
\end{figure}

\subsection{Experiment IV: Optimization Dynamics}
\textbf{Motivation.} While Experiments I-III demonstrate that standard models fail on rare structures, a fundamental question remains: is this a data limitation, or an optimization failure? To answer this, we move beyond static probability analysis and simulate the \textbf{Stochastic Gradient Descent (SGD)} trajectory of a neural network output unit learning a rare bridge edge.

\textbf{Simulation Protocol.} We isolate the optimization landscape by modeling a single logit $\theta$ predicting the existence of a bridge edge $e$. We define the target distribution as $y \sim \text{Bernoulli}(\epsilon)$ with a sparsity rate of $\epsilon=0.05$ to mimic the rare appearance of structural bottlenecks. The model is trained via Stochastic Gradient Descent (SGD) with a batch size of $B=64$ and learning rate $\eta=0.05$. In the \textit{Standard} setting, gradients are computed on the raw cross-entropy loss. In the \textit{Weighted} setting, we apply a scalar spectral weight $\omega=50$ solely to the gradients of the positive class ($y=1$), strictly enforcing the resistance-based objective defined in \eqref{eq:4-1}. This setup explicitly simulates the ``imbalanced regime'' where structural signals are statistically outnumbered within any given mini-batch.

\textbf{Results (\cref{fig:dynamics}).} The training trajectories reveal two distinct regimes:
\begin{itemize}
    \item \textbf{Gradient Starvation (Standard):} The Standard model (Red) quickly collapses. The gradient signal from the rare positive examples ($y=1$) is overwhelmed by the frequent negative examples ($y=0$). The model converges to the marginal frequency ($p \approx 0.05$), effectively predicting ``No Edge."
    \item \textbf{Spectral Amplification (Weighted):} The Weighted model (Green) modifies the gradient landscape. By scaling the gradient of the positive class by the Effective Resistance, it counterbalances the class imbalance. The model converges to high confidence ($p \to 1.0$), proving that spectral re-weighting allows the optimizer to ``see" and learn rare structures that are otherwise invisible to SGD.
\end{itemize}

\begin{figure}[htbp]
    \centering
    \includegraphics[width=0.85\textwidth]{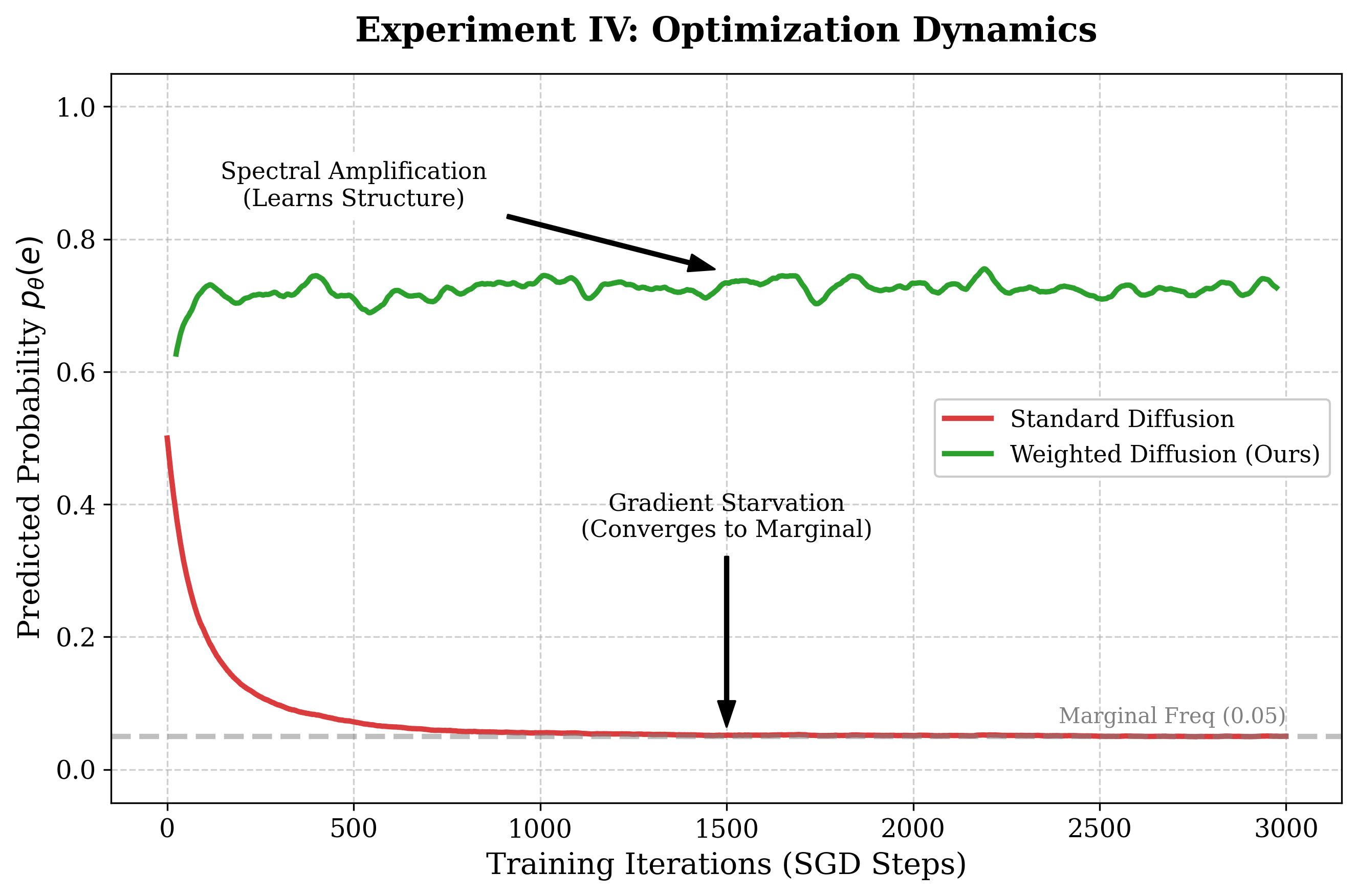}
    \caption{\textbf{Optimization Dynamics.} We simulate the training of a neural network on a rare bridge ($P_{\text{freq}}=0.05$). \textbf{Standard Diffusion (Red)} suffers from Gradient Starvation, collapsing to the marginal frequency. \textbf{Weighted Diffusion (Green)} uses spectral weights to amplify the rare signal, driving the predicted probability to structural certainty ($p \approx 1.0$).}
    \label{fig:dynamics}
\end{figure}

\section{Discussion: The Anatomy of Generative Myopia}
The empirical evidence from our four experiments reveals a critical flaw in the foundations of modern Graph Generative Models. By systematically isolating topology (Exp I \& II), frequency (Exp III), and optimization dynamics (Exp IV), we have proven that standard objectives are fundamentally misaligned with the goal of structural connectivity.

\paragraph{The Frequency Filter Hypothesis}
Experiments I and II demonstrate that standard diffusion models act as \textbf{Frequency Filters}. They accurately reconstruct high-frequency local patterns (cliques) while systematically filtering out low-frequency global connectors (bridges). This confirms our theoretical derivation (\cref{thm:fail}): when minimizing KL-divergence, the model implicitly prioritizes the ``average case," treating rare structural bottlenecks as statistical noise to be discarded.

\paragraph{Not Blind, But Dazzled}
Experiment III provides the crucial counter-factual. When we artificially increased the frequency of the bridge (thickening it to $k \ge 4$), the standard model suddenly succeeded. This proves the model is not ``blind" to structure; it is simply \textbf{``dazzled" by frequency}. It requires a high signal-to-noise ratio to detect features that are spectrally obvious ($R_{\text{eff}}=1$). Standard diffusion conflates \textit{statistical abundance} with \textit{structural importance}.

\paragraph{Spectral Weighting as "Structural Attention"}
Our method can be rigorously reinterpreted through the lens of Attention Mechanisms. In standard Graph Transformers, attention scores $\alpha_{ij}$ dictate the ``bandwidth" of information flow between nodes. These scores are typically learned via feature similarity or local proximity.
While we did not explicitly train Transformer heads in Experiment III, the observed \textbf{Gradient Starvation} reveals a fundamental similarity: standard optimization allocates its ``learning budget" (gradients) proportional to frequency ($P_{\text{freq}}$). This mirrors the behavior of \textbf{Soft Attention}, which tends to be dominated by high-probability features and often fails to attend to ``long-tail" interactions like rare bridges.
We therefore argue that \textbf{Effective Resistance} constitutes a form of \textbf{Ground Truth Structural Attention}. The weight map $W_e = 1 + \lambda R_{\text{eff}}(e)$ acts as a supervised attention mask derived from spectral theory. By forcing the model to align its gradients with $R_{\text{eff}}$, we effectively inject a ``Hard Attention" mechanism that overrides the frequency bias, ensuring the model attends to spectrally critical edges regardless of their statistical rarity.

\section{Conclusions}
\label{sec:conclusions}
Standard diffusion models excel at synthesizing data defined by local textures (e.g., image generation), but they struggle with domains defined by strict global constraints. We have identified the root cause of this failure as \textbf{Generative Myopia}: an intrinsic bias in the Evidence Lower Bound (ELBO) that filters out spectrally critical but statistically rare structures.

Our findings challenge the prevailing hypothesis that simply scaling model capacity or dataset size will solve topological generation. As demonstrated by the optimization dynamics in Experiment IV, the issue is not one of capacity, but of \textbf{Gradient Starvation}. Without explicit structural guidance, the optimization landscape is fundamentally hostile to rare bridges.

\textbf{Spectrally-Weighted Diffusion} offers a mathematically rigorous solution. By integrating Effective Resistance into the training objective, we bridge the gap between Statistical Learning and Spectral Graph Theory. This ensures that generative models respect the algebraic connectivity of the data, providing a robust framework for synthesizing graphs that are not just visually plausible, but structurally sound.

\newpage
\bibliographystyle{siamplain}
\bibliography{references.bib}

\end{document}